\documentclass{article}

\usepackage[preprint]{neurips_2021}

\usepackage[utf8]{inputenc} 
\usepackage[T1]{fontenc}    
\usepackage{hyperref}       
\usepackage{url}            
\usepackage{booktabs}       
\usepackage{amsfonts}       
\usepackage{nicefrac}       
\usepackage{microtype}      
\usepackage{xcolor}         

\usepackage{amsmath,amssymb,amsfonts}
\usepackage[noend]{algorithmic}
\usepackage{textcomp}
\usepackage{xcolor}
\usepackage{lipsum}
\usepackage{indentfirst}
\usepackage{array}
\usepackage{mathrsfs}
\usepackage{wrapfig}
\usepackage{subcaption}
\usepackage{bm}
\usepackage{verbatim}
\usepackage{algorithm}
\usepackage{bbm}
\usepackage{enumerate}
\usepackage{dsfont}
\usepackage{mathtools}
\usepackage{tabularx}
\usepackage{multirow}
\usepackage{multicol}
\usepackage{amsthm}

\newcommand{\X}{\mathbf{X}}
\newcommand{\x}{\mathbf{x}}

\newcommand{\Z}{\mathbf{Z}}
\newcommand{\z}{\mathbf{z}}
\newcommand{\sa}{\mathbf{a}}
\newcommand{\A}{\mathbf{A}}
\newcommand{\W}{\mathbf{W}}
\newcommand{\w}{\mathbf{w}}
\usepackage{color}

\DeclareMathOperator*{\argmax}{arg\,max}

\newcommand\independent{\protect\mathpalette{\protect\independenT}{\perp}}
\def\independenT#1#2{\mathrel{\rlap{$#1#2$}\mkern2mu{#1#2}}}

\makeatletter
  \newcommand\figcaption{\def\@captype{figure}\caption}
  \newcommand\tabcaption{\def\@captype{table}\caption}
\makeatother

\newtheorem{theorem}{Theorem}

\newtheorem{definition}{Definition}
\newtheorem{corollary}{Corollary}
\newtheorem{assumption}{Assumption}
\newtheorem{proposition}{Proposition}
\newtheorem{theorem*}{Theorem}

\title{Achieving Counterfactual Fairness for Causal Bandit}

\author{
Wen Huang, Lu Zhang, Xintao Wu \\
University of Arkansas\\
\{wenhuang, lz006, xintaowu\}@uark.edu
}

\begin{document}

\maketitle

\begin{abstract}
In online recommendation, customers arrive in a sequential and stochastic manner from an underlying distribution and the online decision model recommends a chosen item for each arriving individual based on some strategy. We study how to recommend an item at each step to maximize the expected reward  while achieving user-side fairness for customers, i.e., customers  who share similar profiles will receive a similar reward regardless of their sensitive attributes and items being recommended. By incorporating causal inference into bandits and adopting soft intervention to model the arm selection strategy, we first propose the d-separation based UCB algorithm (D-UCB) to explore the utilization of the d-separation set in reducing the amount of exploration needed to achieve low cumulative regret. Based on that, we then propose the fair causal bandit (F-UCB) for achieving the counterfactual individual fairness. Both theoretical analysis and empirical evaluation demonstrate effectiveness of our algorithms.
\end{abstract}

\section{Introduction}

Fairness in machine learning has been a research subject with rapid growth recently. Many different definitions of fairness have been designed to fit different settings, e.g., equality of opportunity and equalized odds \cite{hardt2016equality}, direct and indirect discrimination \cite{zhang2017causal,zhang2018fairness,chiappa2018path}, counterfactual fairness \cite{kusner2017counterfactual,russell2017worlds,wu2019counterfactual},
and path-specific counterfactual fairness \cite{wu2019pcfairness}. Although there are many works focusing on fairness in personalized recommendation \cite{cseli2018algorithmic,liu2017calibrated,zhu2018fairness}, how to achieve individual fairness in bandit recommendation still remains a challenging task.

We focus on online recommendation, e.g., customers are being recommended items, and consider the setting where customers arrive in a sequential and stochastic manner from an underlying distribution and the online decision model recommends a chosen item for each arriving individual based on some strategy.
The challenge here is how to choose the arm at each step to maximize the expected reward  while achieving user-side fairness for customers, i.e., customers  who share similar profiles will receive similar rewards regardless of their sensitive attributes and items being recommended.

Recently researchers have started taking fairness and discrimination into consideration in the design of personalized recommendation algorithms \cite{cseli2018algorithmic,liu2017calibrated,zhu2018fairness,DBLP:conf/nips/JosephKMR16,DBLP:conf/aies/JosephKMNR18,DBLP:conf/icml/JabbariJKMR17,burke2017multisided,burke2018balanced,ekstrand2018exploring}.
Among them, \cite{DBLP:conf/nips/JosephKMR16} was the first paper of studying fairness in classic and contextual bandits. It defined fairness with respect to one-step rewards and introduced a notion of meritocratic fairness,
i.e., the algorithm should never place higher selection probability on a less qualified arm (e.g., job applicant) than on a more qualified arm.
The following works along this direction include \cite{DBLP:conf/aies/JosephKMNR18} for infinite and contextual bandits, \cite{DBLP:conf/icml/JabbariJKMR17} for reinforcement learning,  \cite{liu2017calibrated} for the simple stochastic bandit setting with calibration based fairness.
However, all existing works require some fairness constraint on arms at every round of the learning process, which is different from our user-side fairness setting. One recent work
\cite{huang2020achieving} focused on achieving user-side fairness in bandit setting, but it
only purposed a heuristic way to achieve correlation based group level fairness and didn't incorporate causal inference and counterfactual fairness into bandits.

By incorporating causal inference into bandits, we first propose the d-separation based upper confidence bound bandit algorithm (D-UCB), based on which we then propose the fair causal bandit (F-UCB) for achieving the counterfactual individual fairness. Our work is inspired by recent research on causal bandits \cite{DBLP:conf/nips/LattimoreLR16,DBLP:conf/icml/SenSDS17,DBLP:conf/nips/LeeB18,DBLP:conf/aaai/LeeB19,lu2020regret}, which studied how to learn optimal interventions sequentially by representing the relationship between interventions and outcomes as a causal graph along with associated conditional distributions. For example, \cite{lu2020regret} developed the causal UCB (C-UCB) that exploits the causal relationships between the reward and its direct parents.
However, different from previous works, our algorithms adopt soft intervention \cite{DBLP:conf/aaai/CorreaB20} to model the arm selection strategy and leverage the d-separation set identified from the underlying causal graph, thus greatly reducing the amount of exploration needed to achieve low cumulative regret. We show that our D-UCB achieves $\Tilde{O}(\sqrt{|\W| \cdot T})$ regret bound where $T$ is the number of iterations and
$\W$ is a set that d-separates arm/user features and reward $R$ in the causal graph.
As a comparison, the C-UCB achieves $\Tilde{O}(\sqrt{|Pa(R)| \cdot T})$ where $Pa(R)$ is the parental variables of $R$ that is a trivial solution of the d-separation set.
In our F-UCB, we further achieve counterfactual fairness in each round of exploration. Counterfactual fairness requires the expected reward an individual would receive keeps the same if the individual's sensitive attribute were changed to its counterpart.  The introduced counterfactual reward combines two interventions, a soft intervention on the arm selection and a hard intervention on the sensitive attribute. The F-UCB achieves counterfactual fairness in online recommendation by picking arms from a subset of arms at each round in which all the arms satisfy counterfactual fairness constraint. Our theoretical analysis shows F-UCB achieves $\Tilde{O}(\frac{\sqrt{|\mathbf{W}|T}}{\tau - \Delta_{\pi_0}})$ cumulative regret bound where $\tau$ is the fairness threshold and $\Delta_{\pi_0}$  denotes the maximum fairness discrepancy of a safe policy $\pi_0$, i.e., a policy that is fair across all rounds.

We conduct experiments on the Email Campaign data \cite{lu2020regret} whose results show the benefit of using the d-separation set from the causal graph. Our D-UCB incurs less regrets than two baselines, the classic UCB which does not leverage any causal information as well as the C-UCB.
In addition, we validate numerically that our F-UCB maintains good performance while satisfying counterfactual individual fairness in each round. On the contrary, the baselines fail to achieve fairness with significant percentages of recommendations violating fairness constraint. We further conduct experiments on the Adult-Video dataset and compare our F-UCB with  another user-side fair bandit algorithm Fair-LinUCB \cite{huang2020achieving}. The results demonstrate the advantage of our causal based fair bandit algorithm on achieving individual level fairness in online recommendation.

\section{Background}
\label{sec:background}

Our work is based on Pearl's structural causal models \cite{pearl2009causality} which describes the causal mechanisms of a system as a set of structural equations.

\begin{definition}[Structural Causal Model (SCM) \cite{pearl2009causality}]\label{def:cm}
A causal model $\mathcal{M}$ is a triple $\mathcal{M} = \langle \mathbf{U},\mathbf{V},\mathbf{F} \rangle$ where
1) $\mathbf{U}$ is a set of hidden contextual variables that are determined by factors outside the model; 	
2) $\mathbf{V}$ is a set of observed variables that are determined by variables in $\mathbf{U}\cup\mathbf{V}$;
3) $\mathbf{F}$ is a set of equations mapping from $\mathbf{U}\times \mathbf{V}$ to $\mathbf{V}$. Specifically, for each $V\in \mathbf{V}$, there is an equation $f_{V}\in \mathbf{F}$ mapping from $\mathbf{U}\times (\mathbf{V}\backslash V)$ to $V$, i.e., $v = f_{V}(Pa(V),\mathbf{u}_{V})$, where $Pa(V)$ is a realization of a set of observed variables called the parents of $V$, and $\mathbf{u}_{V}$ is a realization of a set of hidden variables.
\end{definition}

If all hidden variables in $\mathbf{U}$ are assumed to be mutually independent, then the causal model is called a Markovian model; otherwise it is called a semi-Markovian model.
In this paper, we assume the Markovian model when conducting causal inference.

Quantitatively measuring causal effects is facilitated with the $do$-operator \cite{pearl2009causality}, which simulates the physical interventions that force some variable  to take certain values. Formally, the intervention that sets the value of $X$ to $x$ is denoted by $do(x)$. In a SCM, intervention $do(x)$ is defined as the substitution of equation $x=f_{X}(Pa(X),\mathbf{u}_{X})$ with constant $X=x$.
For an observed variable $Y$ other than $X$,
its variant under intervention $do(x)$ is denoted by $Y(x)$.
The distribution of $Y(x)$, also referred to as the post-intervention distribution of $Y$, is denoted by $P(Y(x))$. The soft intervention (also known as the conditional action, policy intervention)  extends the hard intervention such that it forces variable $X$ to take a new functional relationship in responding to some other variables \cite{DBLP:conf/aaai/CorreaB20}.
Denoting the soft intervention by $\pi$, the post-interventional distribution of $X$ given its parents is denoted by $P_{\pi}(X|Pa(X))$. More generally, the new function could receive as inputs the variables other than
the original parents $Pa(X)$, as long as they are not the descendants of $X$. The distribution of $Y$ after performing the soft intervention is denoted by $P(Y(\pi))$.

With intervention, the counterfactual effect measures the causal effect while the intervention is performed conditioning on only certain individuals or groups specified by a subset of observed variables $\mathbf{O}=\mathbf{o}$. Given a context $\mathbf{O} \!=\! \mathbf{o}$, the counterfactual effect of the value change of $X$ from $x_1$ to $x_2$ on $Y$ is given by
$\mathbb{E}[Y(x_2)|\mathbf{o}] - \mathbb{E}[Y(x_1)|\mathbf{o}]$.

Each causal model $\mathcal{M}$ is associated with a causal graph $\mathcal{G}=\langle \mathbf{V},\mathbf{E} \rangle$, where $\mathbf{V}$ is a set of nodes and $\mathbf{E}$ is a set of directed edges.
Each node in $\mathcal{G}$ corresponds to a variable $V$ in $\mathcal{M}$. Each edge, denoted by an arrow $\rightarrow$, points from each member of $Pa(V)$ toward $V$ to represent the direct causal relationship specified by equation $f_{V}(\cdot)$.
The well-known d-separation criterion \cite{spirtes2000causation} connects the causal graph with conditional independence.
\begin{definition}[\textit{d}-Separation \cite{spirtes2000causation}]\label{def:d}
Consider a causal graph $\mathcal{G}$. $\mathbf{X}$, $\mathbf{Y}$ and $\mathbf{W}$ are disjoint sets of attributes. $\mathbf{X}$ and $\mathbf{Y}$ are d-separated by $\mathbf{W}$ in $\mathcal{G}$, if and only if $\mathbf{W}$ blocks all paths from every node in $\mathbf{X}$ to every node in  $\mathbf{Y}$. A path $p$ is said to be blocked by  $\mathbf{W}$ if and only if:
1) $p$ contains a chain $i\rightarrow m\rightarrow j$ or a fork $i\leftarrow m \rightarrow j$ such that the middle node $m$ is in $\mathbf{W}$, or 2) $p$ contains an collider $i\rightarrow m\leftarrow j$ such that the middle node $m$ is not in $\mathbf{W}$ and no descendant of $m$ is in $\mathbf{W}$.
\end{definition}

\section{Achieving Counterfactual Fairness in Bandit}

In this section, we present our D-UCB and F-UCB bandit algorithms.
The online recommendation is commonly modeled as a contextual multi-armed bandit problem, where each customer is a ``bandit player'', each potential item $a$ has a feature vector $\sa\in \mathcal{A}$ and there are a total number of $k$ items\footnote{We use $\sa$ to represent the feature vector of item/arm $a$, and they may be used interchangeably when the context is unambiguous.}. For each customer arrived at time $t\in [T]$ with feature vector $\mathbf{x}_t\in \mathcal{X}$, the algorithm recommends an item with features $\sa$ based on vector $\mathbf{x}_{t,a}$ which represents the concatenation of the  user and the item feature vectors ($\mathbf{x}_{t}$, $\sa$), observes the reward $r_t$ (e.g., purchase), and then updates its recommendation strategy with the new observation. There may also exist some intermediate features (denoted by $\mathbf{I}$) that are affected by the recommended item and influence the reward, such as the user feedback about relevance and quality.

\subsection{Modeling Arm Selection via Soft Intervention}

In bandit algorithms, we often choose an arm that maximizes  the expectation of the conditional reward, $a_t = \argmax_{a} {\mathbb{E}[R| \mathbf{x}_{t,a}]}$. The arm selection strategy could be implemented by a functional mapping from $\mathcal{X}$ to $\mathcal{A}$, and after each round the parameters in the function get updated with the newest observation tuple.

We advocate the use of the causal graph and soft interventions as a general representation of any bandit algorithm. We consider the causal graph $\mathcal{G}$, e.g., as shown in Figure \ref{fig:cbandit}, where $\A$ represents the arm features, $\X$ represents the user features, $R$ represents the reward, and $\mathbf{I}$ represents some intermediate features between $\A$ and $R$. Since the arm selection process could be regarded as the structural equation of $\X$ on $\A$,
\begin{wrapfigure}{r}{0.35\textwidth}\centering
\vspace*{-2.0ex}
\begin{minipage}{0.35\textwidth}
\small
\centering
   \includegraphics[scale=1.0]{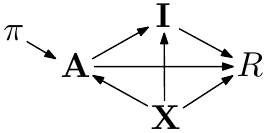}
       \caption{Graph structure for contextual bandit recommendation. $\pi$  denotes the soft intervention conducted on arm selection.}
       \vspace*{-2.0ex}
   \label{fig:cbandit}
\end{minipage}
\end{wrapfigure}
we treat $\X$ as $\A$'s parents. Then, the reward $R$ is influenced by the arm selection, the contextual user features, as well as some intermediate features, so all the three factors are parents of $R$. In this setting, it is natural to treat the update of the arm selection policy as a soft intervention $\pi$ performed on the arm features $\A$. Each time when an arm selection strategy is learned, the corresponding soft intervention is considered to be conducted on $\A$ while user features $\mathbf{X}$ and all other relationships in the causal graph are unchanged.

There are several advantages of modeling arm selection learning using the soft intervention. First, it can capture the complex causal relationships between context and reward without introducing strong assumptions, e.g., linear reward function, or Gaussian/Bernoulli prior distribution, which are often not held in practice. Second, it is flexible in terms of the functional form. For example, it can be of any function type, and it can be independent or dependent upon the target variable's existing parents and can also include new variables that are not the target variable's parents. Third, the soft intervention can be either deterministic, i.e., fixing the target variable to a particular constant, or stochastic, i.e., assigns to the target variable a distribution with probabilities over multiple states. As a result, most existing and predominant bandit algorithms could be described using this framework. Moreover, based on this framework we could propose new bandit algorithms by adopting different soft interventions.

Formally, let $\Pi_t$ be the arm selection policy space at time $t \in [T]$, and $\pi \in \Pi_t$ be a specific policy. The implementation of policy $\pi$ is modeled by a soft intervention. Denoting by $R(\pi)$ the post-interventional value of the reward after performing the intervention, the expected reward under policy $\pi$, denoted by $\mu_{\pi}$, is given by $\mathbb{E}[R(\pi)| \mathbf{x}_t]$. According to the $\sigma$-calculus \cite{DBLP:conf/aaai/CorreaB20}, it can be further decomposed as follows:
\begin{equation}\label{eq:reward-expectation}
\begin{split}
\mu_{\pi} = \mathbb{E}[R(\pi)| \mathbf{x}_{t}] = \sum_{\sa} P_{\pi}(\sa|\mathbf{x}_{t}) \cdot \mathbb{E}[R(\sa)|\mathbf{x}_{t}] = \mathbb{E}_{\sa \sim \pi} \left[ \mathbb{E}[R(\sa)|\mathbf{x}_{t}] \right]
\end{split}
\end{equation}
where $P_{\pi}(\sa|\mathbf{x}_{t})$ is a distribution defined by policy $\pi$.
As can be seen, once a policy is given, the estimation of $\mu_{\pi}$ depends on the estimation of $\mathbb{E}[R(\sa)|\mathbf{x}_{t}]$ (denoted by $\mu_{a}$). Note that $\mu_{a}$ represents the expected reward when selecting an arm $a$, which is still a post-intervention quantity and needs to be expressed using observational distributions in order to be computable. In the following, we propose a d-separation based estimation method and based on which we develop our D-UCB algorithm. For the ease of representation, our discussions in Sections 3.2, 3.3 and 3.4 assume deterministic policies but in principle the above framework could be applied to stochastic policies as well.

\subsection{D-UCB Algorithm}

Let $\W\subseteq \A\cup\X\cup\mathbf{I}$ be a subset of nodes that d-separates reward $R$ from features $(\A\cup\X) \backslash \W$ in the causal graph.
Such set always exists since $\A\cup\X$ and $Pa(R)$ are trivial solutions.
Let $\Z = \W \backslash (\A\cup \X)$.
Using the do-calculus \cite{pearl2009causality}, we can decompose $\mu_{a}$ as follows.
\begin{equation}\label{eq:mua}
\begin{split}
    \mu_{a} & = \mathbb{E}[R|do(\sa),\mathbf{x}_{t}] = \sum_{\mathbf{Z}} \mathbb{E}[R|\mathbf{z},do(\sa),\x_t] P(\mathbf{z}| do(\sa),\mathbf{x}_{t}) \\
    & = \sum_{\mathbf{Z}} \mathbb{E}[R|\mathbf{z},\sa,\x_t] P(\mathbf{z}| \sa, \mathbf{x}_{t})
    = \sum_{\mathbf{Z}} \mathbb{E}[R|\mathbf{z},\sa,\x_t]  P(\mathbf{z}| \x_{t,a})
     = \sum_{\mathbf{Z}} \mathbb{E}[R|\mathbf{w}]  P(\mathbf{z}| \x_{t,a})
\end{split}
\end{equation}
where the last step is due to the d-separation.
Similarly to \cite{lu2020regret}, we assume that distribution $P(\mathbf{z}| \x_{t,a})$ is known based on previous knowledge that was used to build the causal graph.
Then, by using a sample mean estimator (denoted by $\hat{\mu}_{\mathbf{w}}(t)$) to estimate $\mathbb{E}[R|\mathbf{w} ]$ based on the observational data up to time $t$, the estimated reward mean is given by
\begin{align}
    \hat{\mu}_{\pi}(t) &=
     \mathbb{E}_{\sa \sim \pi} \left[ \sum_{\Z} \hat{\mu}_{\mathbf{w}}(t) \cdot P(\z| \mathbf{x}_{t,a}) \right]
\label{estimated mean}
\end{align}

Subsequently, we propose a causal bandit algorithm based on d-separation, called D-UCB. Since there is always uncertainty on the reward given a specific policy, in order to balance exploration and exploitation we follow the rule of optimistic in the face of uncertainty (OFU) in D-UCB algorithm. The policy taken at time $t$ will lead to the highest upper confidence bound of the expected reward, which is given by
\begin{equation}\label{eq:pi_t}
\pi_t = \argmax_{\pi \in \Pi_t}\mathbb{E}_{\sa \sim \pi}[UCB_{a}(t)]
\end{equation}
\begin{equation}\label{eq:ucba}
  UCB_{a}(t) =  \sum_{\Z} UCB_{\mathbf{w}}(t) P(\mathbf{z}|\x_{t,a})
\end{equation}
Since $\hat{\mu}_{\mathbf{w}}(t)$ is an unbiased estimator and the error term of the reward is assumed to be sub-Gaussian distributed, the $1-\delta$ upper confidence bound of $\mu_{\mathbf{w}}(t)$  is given by
\begin{equation}
 UCB_{\mathbf{w}}(t) = \hat{\mu}_{\mathbf{w}}(t) + \sqrt{\frac{2\log(1/\delta)}{1 \vee N_{\mathbf{w}}(t)}}
\end{equation}
After taking the policy, we will have new observations on $r_t$ and $\w_t$. The sample mean estimator is then updated accordingly:
\begin{equation}\label{eq:update}
    \hat{\mu}_{\mathbf{w}}(t) =  \frac{1}{T_{\mathbf{w}}(t)} \sum_{k=1}^{t} r_t  \mathbbm{1}_{\mathbf{w}_k = \mathbf{w}} ~~ \textrm{where}~~  T_{\mathbf{w}}(t) = \sum_{k=1}^{t} \mathbbm{1}_{\mathbf{w}_k = \mathbf{w}}
\end{equation}

We hypothesize that the choice of d-separation set $\W$ would significantly affect the regret of the D-UCB. To this end, we analyze the upper bound of the cumulative regret $\mathcal{R}_T$.
The following theorem shows that, the regret upper bound depends on the domain size of d-separation set $\W$.

\begin{theorem}[Regret bound of D-UCB]\label{thm:rb}
Given a causal graph $\mathcal{G}$, with probability at least $1-2\delta T |\mathbf{W}| - \exp(- \frac{|\mathbf{W}| \log^3(T)}{32\log(1/\delta)})$, the regret of D-UCB is bounded by
\[   \mathcal{R}_T \leq \sqrt{|\mathbf{W}| T\log(T)}log(T) + \sqrt{32|\mathbf{W}|T \log(1/\delta)}   \]
 where $|\mathbf{W}|$ is the domain space of set $\W$.
\label{causalucb_1}
\end{theorem}

\begin{proof}[Proof Sketch]
The proof of Theorem \ref{causalucb_1} follows the general regret analysis framework of the UCB algorithm \cite{auer2002finite}. By leveraging d-separation decomposition of the expected reward,  we split the cumulative regret into two terms and bound them separately. Since there are less terms to traverse when summing up and bounding the uncertainty caused by exploration-exploitation strategy,  D-UCB is supposed to obtain lower regret than the original UCB algorithm and  C-UCB algorithm.
By setting $\delta = 1/T^2$, it is easy to show that D-UCB algorithm achieves $\Tilde{O}(\sqrt{|\W| \cdot T})$ regret bound.
Please refer to Appendix B in the supplementary file for proof details.
\end{proof}

\begin{algorithm}[t]\small
	\caption{D-UCB: Causal Bandit based on d-separation}
	\begin{algorithmic}[1]
		\STATE  \text{Input:}  Policy space $\Pi$, confidence level parameter $\delta$, original causal Graph $\mathcal{G}$ with domain knowledge
		\STATE Find the d-separation set $\mathbf{W}$ with minimum subset $\Z$ in terms of domain space.
	    \FOR {$t = 1,2,3,...,T$}
	    \smallskip
	    \STATE Obtain the optimal policy $\pi_t$ following Eq.~\eqref{eq:pi_t}.
	    \STATE Take action $\sa_t \sim \pi_t $ and observe a real-valued payoff $r_{t}$ and a d-separation set value $\mathbf{w}_t$.
	    \STATE Update $\hat{\mu}_{\mathbf{w}}(t)$ for all $\w \in \W$ following Eq.~\eqref{eq:update}.
	    \ENDFOR
	\end{algorithmic}
	\label{alg:ducb}
\end{algorithm}

Algorithm \ref{alg:ducb} shows the pseudo code of the D-UCB. In Line 2, according to Theorem \ref{thm:rb}, we first determine the d-separation set $\mathbf{W}$ with the minimum domain space. In Line 4 we leverage causal graph and the observational data up to time $t$   to find the optimal policy $\pi_t = \argmax_{\pi \in \Pi_t}\mathbb{E}_{\sa \sim \pi}[UCB_{a}(t)]$. In Line 5, we take action $\sa_t \sim \pi_t $ and observe a real-valued payoff $r_t$, and in Line 6, we update the observational data with $\sa_t$ and $r_t$.

{\noindent\bf Remark.} Determining the minimum d-separation set has been well studied in causal inference \cite{geiger1990d}.  We leverage the algorithm of finding a minimum cost separator \cite{tian1998finding} to identify $\W$. The discovery procedure usually requires the complete knowledge of the causal graph. However, in the situation where the d-separation set to be used as well as the associated conditional distributions $P(\mathbf{z}|\x_{t,a})$ are given, the remaining part of the algorithm will work just fine without the causal graph information. Moreover, the assumption of knowing $P(\mathbf{z}|\mathbf{x_{t,a}})$ follows recent research works on causal bandit. Generalizing the causal bandit framework to partially/completely unknown causal graph setting is a much more challenging but important task.  A recent work \cite{lu2021causal} tries to generalize causal bandit algorithm based on causal  trees/forests structure.

To better illustrate the long-term regret of causal bandit algorithm, suppose the set $\mathbf{A} \cup \mathbf{U} \cup \mathbf{I}$ includes $N$  variables that are related to the reward and the d-separation set $\mathbf{W}$ includes $n$ variables. If each of the variable takes on 2 distinct values,  the number of deterministic policies can be as large as $2^N$  for traditional bandit algorithm, leading to a $\mathcal{O}(\sqrt{2^NT})$ regret bound. On the other hand, our proposed causal algorithms exploit the knowledge of the d-separation set $\mathbf{W}$ and achieves $\mathcal{O}(\sqrt{2^nT})$ regret, which implies a significant reduction regarding to the regret bound if $n << N$. If the number of arm candidates is much smaller than the domain space of $\mathbf{W}$, our bound analysis could be easily adjusted to this case using a subspace of $\mathbf{W}$ that corresponds to the arm candidates.

\subsection{Counterfactual Fairness}

Now, we are ready to present our fair UCB algorithm.
Rather than focusing on the fairness of the item being recommended (e.g., items produced by small companies have similar chances of being recommended as those from big companies), we focus on the user-side fairness in terms  of reward, i.e., individual users  who share similar profiles will receive similar rewards regardless of their sensitive attributes and items being recommended such that they both benefit from the  recommendations equally. To this end, we adopt counterfactual fairness as our fairness notion.

Consider a sensitive attribute $S\in \X$ in the user's profile.
Counterfactual fairness concerns the expected reward an individual would receive assuming that this individual were in different sensitive groups. In our context, this can be formulated as the counterfactual reward $\mathbb{E}[R(\pi,s^{*})|\mathbf{x}_t]$ where two interventions are performed simultaneously: soft intervention $\pi$ on the arm selection and hard intervention $do(s^{*})$ on the sensitive attribute $S$, while conditioning on individual features $\mathbf{x}_t$. Denoting by $\Delta_{\pi} =
  \mathbb{E}[R(\pi, s^+)| \mathbf{x}_t ]  -  \mathbb{E}[R(\pi, s^-)|\mathbf{x}_t]$ the counterfactual effect of $S$ on the reward, a policy that is counterfactually fair is defined as follows.

\begin{definition}
A policy $\pi$ is counterfactually fair for an individual arrived if $\Delta_{\pi}=0$. The policy is $\tau$- counterfactually fair if $\left |\Delta_{\pi} \right| \leq \tau$
where $\tau$ is the predefined fairness threshold.
 \label{def:cf}
\end{definition}

To achieve counterfactual fairness in online recommendation, at round $t$, we can only pick arms from a subset of arms for the customer (with feature $\mathbf{x}_t$), in which all the arms satisfy counterfactual fairness constraint. The fair policy subspace $\Phi_t \subseteq \Pi_t$ is thus given by $\Phi_t  = \{\pi : \Delta_{\pi} \leq \tau \} $.

However, the counterfactual fairness is a causal quantity that is not necessarily unidentifiable from observational data without the knowledge of structure equations \cite{shpitser2008complete}. In \cite{wu2019counterfactual}, the authors studied the criterion of identification of counterfactual fairness given a causal graph and provided the bounds for unidentifiable counterfactual fairness.
According to Proposition 1 in \cite{wu2019counterfactual}, our counterfactual fairness is identifiable if $\X \backslash \{S\} $ are not descendants of $S$. In this case,
similar to Eq.~\eqref{eq:reward-expectation}, we have that $\mathbb{E}[R(\pi, s^*)| \mathbf{x}_t ] = \mathbb{E}_{\sa \sim \pi} \left[ \mathbb{E}[R(\sa,s^{*})|\mathbf{x}_{t}] \right]$ where $s^* \in \{s^+, s^-\}$. Similar to Eq.~\eqref{eq:mua}, we denote $\mu_{a,s^*} = \mathbb{E}[R(a,s^*)|\mathbf{x}_{t}]$, which can be decomposed using the do-calculus as
\begin{align}
\mu_{a,s^*} = \mathbb{E}[R(a,s^*)|\mathbf{x}_{t}] = \sum_{\Z} \mathbb{E}[R| s^{*},\mathbf{w}\backslash s_t] \cdot  P(\mathbf{z}|s^*, \mathbf{x}_{t,a} \backslash s_t)
\label{cf}
\end{align}
where $\mathbf{w}\backslash s_t$ and $\mathbf{x}_{t,a} \backslash s_t$ represent all values in $\mathbf{w}$ and $\mathbf{x}_{t,a}$ except $s_t$ respectively. Note that $s^*$ is the sensitive attribute value in the counterfactual world which could be different from the observational value $s_t$. The estimated counterfactual reward can be calculated as
\begin{align}
    \hat{\mu}_{a,s^*}(t) = \sum_{\Z} \hat{\mu}_{\mathbf{w}^{*}}(t) \cdot  P(\mathbf{z}|s^*, \mathbf{x}_{t,a} \backslash s_t)
\end{align}
where $\mathbf{w}^{*}=\{s^{*},\mathbf{w}\backslash s_t\}$ and $\hat{\mu}_{\mathbf{w}^{*}}(t)$ is again the sample mean estimator based on the observational data up to time $t$.
The estimated counterfactual discrepancy of a policy is
\begin{equation}\label{estimated cf}
\hat{\Delta}_{\pi}(t) =  \left| \mathbb{E}_{\sa \sim \pi}[\hat{\mu}_{a,s^+}(t) ]  -  \mathbb{E}_{\sa \sim \pi}[\hat{\mu}_{a,s^-}(t)]  \right|
\end{equation}

In the case where
$\mu_{a,s^*}$ is not identifiable, based on Proposition 2 in \cite{wu2019counterfactual} we derive the lower and upper bounds of $\mu_{a,s^*}$ as presented in the following theorem.
\begin{theorem}
Given a causal graph as shown in Figure \ref{fig:cbandit}, if there exists a non-empty set $\mathbf{B} \subseteq \X \backslash \{S\} $ which are descendants of $S$, then
$\mu_{a,s^*}=\mathbb{E}[R(a,s^*)|\mathbf{x}_{t}]$ is bounded by
\begin{equation}
\sum_{\Z}\min_{\mathbf{b}}\{ \mathbb{E}[R|s^{*},\mathbf{w}\backslash s_t] \} \cdot P(\z|\mathbf{x}_{t,a}) \leq \mu_{a,s^*} \leq \sum_{\Z}\max_{\mathbf{b}}\{ \mathbb{E}[R|s^{*},\mathbf{w}\backslash s_t] \} \cdot P(\z|\mathbf{x}_{t,a})
\end{equation}
\end{theorem}
Please refer to Appendix C of the supplementary file for the proof.

\subsection{F-UCB Algorithm}

Taking the estimation error of the counterfactual discrepancy into consideration, we could also use the high probability upper confidence bound of the counterfactual effect to build the conservative fair policy subspace
$\bar{\Phi}_t  = \{\pi : UCB_{\Delta_{\pi}}(t) \leq \tau \}$ where
\begin{equation}
UCB_{\Delta_{\pi}}(t) = \hat{\Delta}_{\pi}(t) + \sum_{\Z}\sqrt{\frac{8\log(1/\delta)}{1 \vee N_{\mathbf{w}}(t)}}P(\mathbf{z}|\mathbf{x}_{t,a})
\label{ucb_pi}
\end{equation}
which is derived based on the fact that the sum of two independent sub-Gaussian random variables is still sub-Gaussian distributed.
Thus, the learning problem can be formulated as the following constrained optimization problem:
\begin{equation}
\min \mathcal{R}_T = \sum_{t=1}^{T} \left( \mathbb{E}_{\sa \sim \pi_t^*}[\mu_{a}] - \mathbb{E}_{\sa \sim \pi_t}[\mu_{a}] \right) ~~ \mathrm{s.t.}  ~~ \forall t, ~ \pi_t\in \bar{\Phi}_t,
\end{equation}
where $\pi_t^*$ is defined as the optimal policy in the policy space $\Pi_t$ at each round, which is the same in D-UCB setting. The Assumption 3 in Appendix A gives the definition of a safe policy $\pi_0$, which refers to a feasible solution under the fair policy subspace at each round, i.e., $\pi_0 \in \Pi_t$ such that $\Delta_{\pi_0} \leq \tau$ for each $t \in [T]$.

This optimization can be solved similarly by following the rule of OFU.
Algorithm \ref{alg:fucb} depicts our fair bandit algorithm called the F-UCB. Different from the D-UCB algorithm, F-UCB only picks arm from $\bar{\Phi}_t$ at each time $t$.
In Line 5, we compute the estimated reward mean and the estimated fairness discrepancy. In Line 6, we determine the fair policy subspace $\bar{\Phi}_t$, and in Line 7, we find the optimal policy $\pi_t = \argmax_{\pi \in \bar{\Phi}_t}\mathbb{E}_{\sa \sim \pi}[UCB_{a}(t)]$.

\begin{algorithm}[h]\small
	\caption{F-UCB: Fair Causal Bandit}
	\begin{algorithmic}[1]
		\STATE \text{Input:} Policy space $\Pi$, fairness threshold $\tau$,  confidence level parameter $\delta$, original causal Graph $\mathcal{G}$ with domain knowledge
		\STATE Find the d-separation set $\mathbf{W}$ with minimum subset $\Z$ in terms of domain space.
	    \FOR {$t = 1,2,3,...,T$}
        \FOR {$\pi \in \Pi_t$}
        \STATE Compute the estimated reward mean using Eq.~\eqref{estimated mean} and the estimated fairness discrepancy using Eq.~\eqref{estimated cf}.
        \ENDFOR
	    \smallskip
	    \STATE Determine the conservative fair policy subspace $\bar{\Phi}_t$.
	    \STATE Find the optimal policy following Eq.~\eqref{eq:pi_t} within $\bar{\Phi}_t$.
	    \STATE Take action $\sa_t \sim \pi_t $ and observe a real-valued payoff $r_t$ and a d-separation set value $\mathbf{w}_t$.
	    \STATE Update $\hat{\mu}_{\mathbf{w}}(t)$ for all $\w \in \W$.
	    \ENDFOR
	\end{algorithmic}
	\label{alg:fucb}
\end{algorithm}

The following regret analysis shows that, the regret bound of F-UCB is larger than that of D-UCB as expected, and it is still influenced by the domain size of set $\W$.
\begin{theorem}[Regret bound of fair causal bandit]
Given a causal graph $\mathcal{G}$, let $\delta_E = 4|\mathbf{W}|T\delta$ and $\Delta_{\pi_0}$  denote the maximum fairness discrepancy of a safe policy $\pi_0$ across all rounds.
Setting $\alpha_c = 1$ and $\alpha_r = \frac{2}{\tau - \Delta_{\pi_0}}$, with probability at least $1-\delta_E$, the cumulative regret of F-UCB is bounded by:
\[ \mathcal{R}_T  \leq (\frac{2}{\tau - \Delta_{\pi_0}}+1) \times \left( 2\sqrt{2T|\mathbf{W}|\log(1/\delta_E)} + 4\sqrt{T\log(2/\delta_E)\log(1/\delta_E)} \right) \]
\label{cfbound}
\end{theorem}

\begin{proof}[Proof Sketch]
Our derivation of the regret upper bound of F-UCB follows the proof idea of bandits with linear constraints \cite{pacchiano2021stochastic}, where we treat counterfactual fairness as a linear constraint.
By leveraging the knowledge of a feasible fair policy at each round and properly designing the numerical relation of the scale parameters $\alpha_c$ and $\alpha_r$, we are able to synchronously bound the cumulative regret of reward and fairness discrepancy term. Merging these two parts of regret analysis together leads to a unified bound of the F-UCB algorithm.
By setting $\delta_E$ to $1/T^2$ we can show F-UCB  achieves $\Tilde{O}(\frac{\sqrt{|\mathbf{W}|T}}{\tau - \Delta_{\pi_0}})$ long-term regret. The detailed proof is reported in Appendix D of the supplementary file.
\end{proof}

{\noindent\bf Remark.} In Theorem \ref{cfbound}, $\alpha_c$ and $\alpha_r$ refer to the scale parameters that control the magnitude of the confidence interval for sample mean estimators related to reward and fairness term respectively. Appendix \ref{proof_fucb} shows the numerical relation $\alpha_c$ and $\alpha_r$ should satisfy in order to synchronously bound the uncertainty caused by the error terms. The values taken in Theorem \ref{proof_fucb} is one feasible solution with $\alpha_c$ taking the minimum value under the constraint domain space.

The general framework we proposed (Eq. \eqref{eq:reward-expectation}) can be applied to any policy/function class. However, the D-UCB and F-UCB algorithms we proposed still adopt the deterministic policy following the classic UCB algorithm. Thus, the construction of $\bar{\Phi}_t  = \{\pi :UCB_{\Delta_{\pi}}(t) \leq \tau \}$ can be easily achieved as the total number of policies are finite. In this paper we also assume discrete variables, but in principle the proposed algorithms can also be extended to continuous variables by employing certain approximation approaches, e.g., neural networks for estimating probabilities and  sampling approaches for estimating integrals.  However, the regret bound analysis may not apply as $|\mathbf{W}|$ will become infinite in the continuous space.

\section{Experiment}

In this section, we conduct experiments on two datasets and compare the performance of D-UCB and F-UCB with UCB, C-UCB and Fair-LinUCB in terms of the cumulative regret. We also demonstrate the fairness conformance of F-UCB and the violations of other algorithms.

\subsection{Email Campaign Dataset}

We adopt the Email Campaign data as used in previous works \cite{lu2020regret}. The dataset is constructed based on the online advertising process. Its goal is to determine the best advertisement recommendation strategy for diverse user groups to improve their click through ratio (CTR), thus optimize the revenue generated through advertisements.
Figure \ref{fig:email} shows the topology of the causal graph. We use $X_1$, $X_2$, $X_3$ to denote three user profile attributes, \textit{gender}, \textit{age} and \textit{occupation}; $A_1$, $A_2$, $A_3$ to denote three arm features, \textit{product}, \textit{purpose}, \textit{send-time} that could be intervened; $I_1$, $I_2$, $I_3$, $I_4$  to denote \textit{Email body template}, \textit{fitness}, \textit{subject length}, and \textit{user query}; and $R$  to denote the reward that indicates whether users click the advertisement. The reward function is $R = 1/12(I_1 + I_2 + I_3 + A_3) + \mathcal{N}(0,\sigma^2)$, where $\sigma  = 0.1$.
In our experiment, we set $\delta = 1/t^2$ for each $t \in [T]$. In Appendix E.1, we show the domain values of all 11 attributes and their conditional probability tables.

\begin{figure}[htpb]
\centering
\begin{subfigure}{0.45\textwidth}
\centering
  \includegraphics[width=7cm,height=5.4cm]{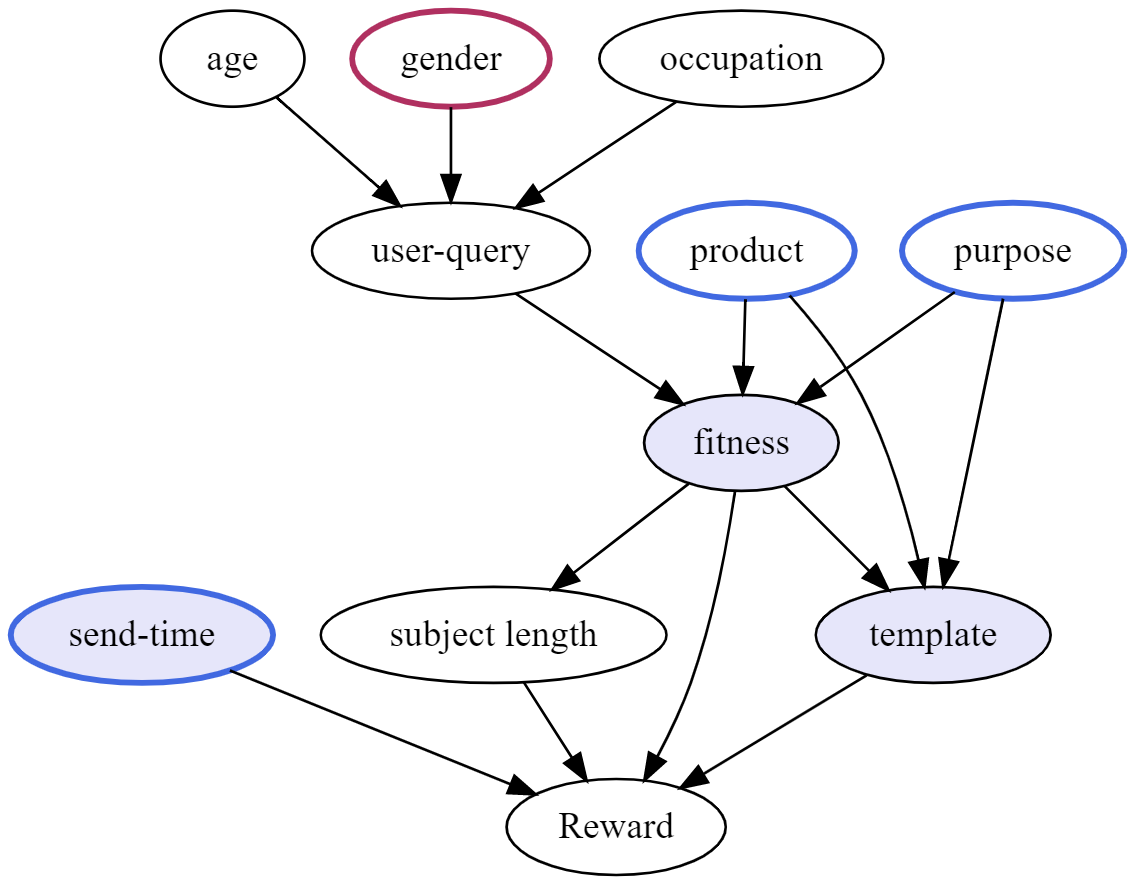}
  \caption{Graph structure under Email Campaign setting.  Nodes with blue frame denote the variables that can be intervened. The node with red frame is the sensitive attribute. Light shaded nodes denote the minimal d-separation set. }
  \label{fig:email}
\end{subfigure}%
\begin{subfigure}{0.6\textwidth}
    \centering
    \includegraphics[width=7cm,height=5.4cm]{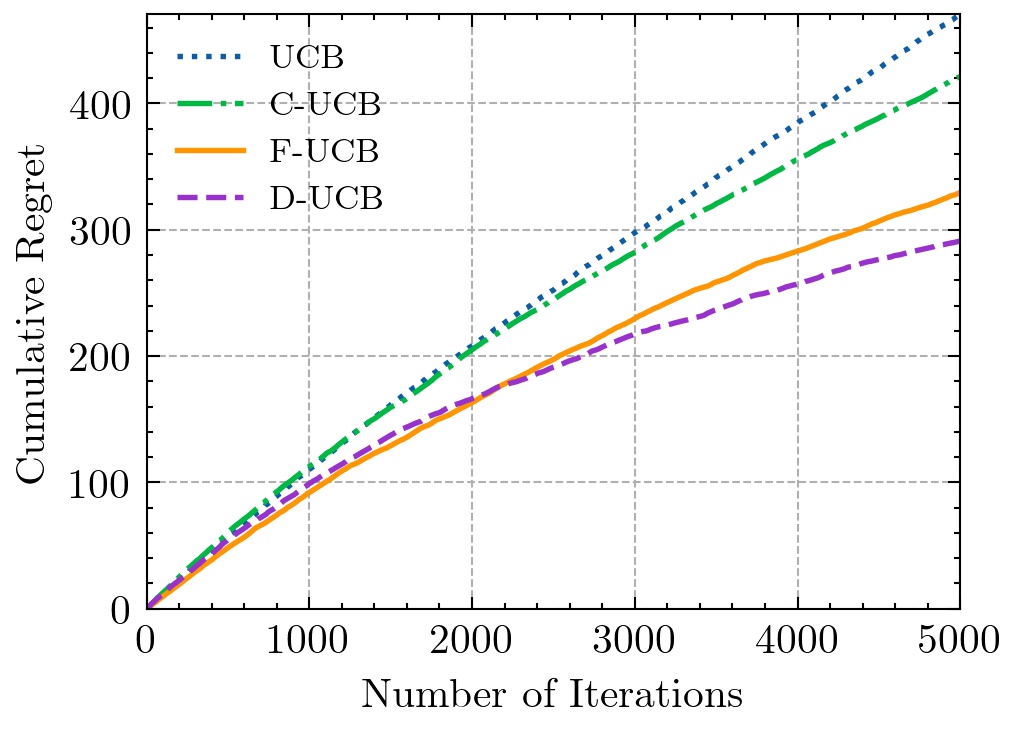}
    \caption{Comparison of bandit algorithms ($\tau = 0.3$ for F-UCB)}
    \label{main result}
\end{subfigure}
\end{figure}

Figure \ref{main result} plots the cumulative regrets of different bandit algorithms along $T$. For each bandit algorithm, the online learning process starts from initialization with no previous observation. Figure \ref{main result} shows clearly all three causal bandit algorithms perform better than UCB. This demonstrates the advantage of applying causal inference in bandits. Moreover, our D-UCB and F-UCB outperform C-UCB, showing the advantage of using d-separation set in our algorithms. The identified  d-separation set  $\W$ (\textit{send time}, \textit{fitness}, and \textit{template}) and the domain space of $\Z$ (\textit{fitness} and \textit{template}) significantly reduce the exploration cost in D-UCB and F-UCB.

{\noindent\bf Remark.} Note that in Figure \ref{main result}, for the first 2000 rounds, F-UCB has lower cumulative regret than D-UCB. A possible explanation is that fair constraint may lead to a policy subspace that contains many policies with high reward. As the number of explorations increase, D-UCB gains more accurate reward estimations for each policy in the whole policy space and eventually outperforms F-UCB.

\begin{table}
\caption{Comparison results for Email Campaign Data}
\centering
\begin{tabular}{|c|c|c|c|c|c|}
\hline
\multirow{2}{*}{$\tau$} & \multirow{2}{*}{\begin{tabular}[c]{@{}c@{}}Cumulative Regret \\ of F-UCB\end{tabular}} & \multicolumn{4}{c|}{Unfair Decisions} \\ \cline{3-6}
                           &                                                                                        & UCB     & C-UCB   & D-UCB   & F-UCB   \\ \hline
0.1                        & 392.12                                                                                 & 3030    & 3176     & 3473         & 0       \\ \hline
0.2                        & 363.55                                                                                 & 1383    & 1487        &  1818       & 0       \\ \hline
0.3                        & 355.21                                                                                 & 482    & 594        & 739         & 0       \\ \hline
0.4                        & 317.80                                                                                 & 141     & 185        &    234     & 0       \\ \hline
0.5                        & 313.89                                                                                 & 18     &  27       &    47     & 0       \\ \hline
\end{tabular}
\label{comparison result}
\end{table}

Table $\ref{comparison result}$ shows how the cumulative regret of F-UCB ($T = 5000$ rounds) varies with the fairness threshold $\tau$. The values in Table \ref{comparison result} (and Table \ref{adult experiment}) are obtained by averaging the results over 5 trials. The larger the $\tau$, the smaller the cumulative regret. In the right block of Table $\ref{comparison result}$, we further report the number of fairness violations of the other three algorithms during the exploration of $T = 5000$ rounds, which demonstrates the need of fairness aware bandits. In comparison, our F-UCB achieves strict counterfactual fairness in every round.

\subsection{Adult-Video Dataset}
We further compare the performance of F-UCB algorithm with Fair-LinUCB \cite{huang2020achieving} on Adult-Video dataset. We follow the settings of \cite{huang2020achieving} by combining two publicly available datasets: Adult dataset and Youtube video dataset. We include in Appendix E.2 detailed information about datasets and experiment. We select 10,000 instances and use half of the data as the offline dataset to construct causal graph and adopt the other half to be user sequence and arm candidates for online recommendation.
The causal graph constructed from the training data is shown in Figure \ref{causal graph}, where $\mathbf{X} = \{ age, sex, race, income\}$ denote user features, $\mathbf{A} = \{ length,ratings,views,comments\}$ denote video features. Bold nodes denote direct parents of the reward and red nodes denote the sensitive attribute. The minimum d-separation set for this graph topology is $\W = \{\textit{age, income, ratings, views}\}$.
The reward function is set as $R = 1/5(age + income + ratings + views) + \mathcal{N}(0, \sigma^2)$, where $\sigma = 0.1$. Following previous section we set $\delta = 1/t^2$ for each $t \in [T]$. The cumulative regret is added up through 5000 rounds.

\begin{figure}
\begin{minipage}[b]{.5\linewidth}
\centering
\includegraphics[width=6.0cm,height=4cm]{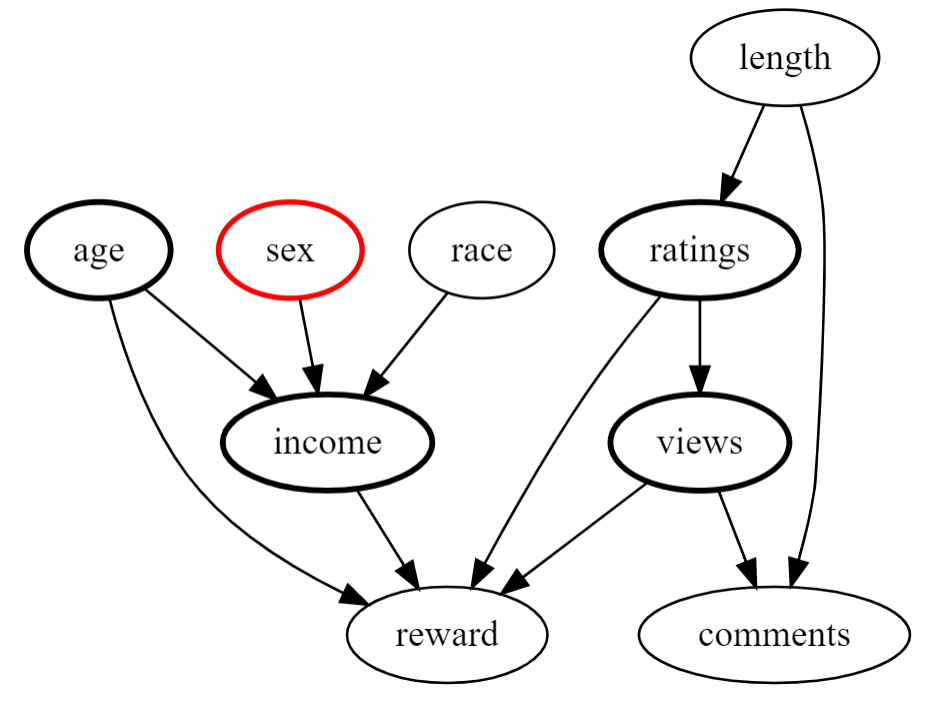}
\caption{Causal graph for adult-video data}
\label{causal graph}
\end{minipage}
\begin{minipage}[b]{.5\linewidth}
\centering
\begin{tabular}{|c|c|c|l|c|l|}
\hline
\multirow{2}{*}{$\tau$} & \begin{tabular}[c]{@{}c@{}} Regret\end{tabular} & \multicolumn{4}{c|}{Unfair Decisions}                         \\ \cline{2-6}
                           & F-UCB                                                                 & \multicolumn{2}{c|}{F-UCB} & \multicolumn{2}{c|}{Fair-LinUCB} \\ \hline
0.1                        & 361.43                                                                & \multicolumn{2}{c|}{0}     & \multicolumn{2}{c|}{2053}            \\ \hline
0.2                        & 332.10                                                                & \multicolumn{2}{c|}{0}     & \multicolumn{2}{c|}{1221}            \\ \hline
0.3                        & 323.12                                                                & \multicolumn{2}{c|}{0}     & \multicolumn{2}{c|}{602}            \\ \hline
0.4                        & 303.32                                                                & \multicolumn{2}{c|}{0}     & \multicolumn{2}{c|}{82}            \\ \hline
0.5                        & 296.19                                                                & \multicolumn{2}{c|}{0}     & \multicolumn{2}{c|}{6}            \\ \hline

\end{tabular}
\tabcaption{Comparison results for Adult-Video Data}
\label{adult experiment}
\end{minipage}
\end{figure}

We  observe from Table \ref{adult experiment} a high volume of unfair decisions made by Fair-LinUCB under strict fairness threshold (nearly forty percent of the users are unfairly treated when $\tau = 0.1$). This implies Fair-LinUCB algorithm can not achieve individual level fairness when conducting online recommendation compared to F-UCB. On the other hand, the cumulative regret for Fair-LinUCB is around 250 over 5000 rounds, which is slightly better than F-UCB. This is because we use the same linear reward setting as  \cite{huang2020achieving} in our experiment and Lin-UCB based algorithm will better catch the reward distribution under this setting.

\section{Related Work}

{\noindent\bf Causal Bandits.}
There have been a few research works of studying how to learn optimal interventions sequentially by representing the relationship between interventions and outcomes as a causal graph along with associated conditional distributions.  \cite{DBLP:conf/nips/LattimoreLR16} introduced the causal bandit problems in which interventions are treated as arms in a bandit problem but their influence on the reward, along with any other observations, is assumed to conform to a known causal graph. Specifically they focus on the setting that observations are only revealed after selecting an intervention (and hence the observed features cannot be used as context) and the distribution of the parents of the reward is known under those interventions.
\cite{DBLP:conf/nips/LeeB18} developed a way to choose an intervention subset based on the causal graph structure as  a brute-force way to apply standard bandit algorithms on all interventions can suffer huge regret.
\cite{DBLP:conf/aaai/LeeB19} studied a relaxed version of the structural causal bandit problem when not all variables are manipulable. \cite{DBLP:conf/icml/SenSDS17} considered best intervention identification via importance sampling. Instead of forcing a node to take a specific value, they adopted soft intervention that  changes the conditional distribution of a node given its parent nodes. \cite{lu2020regret} proposed two algorithms, causal upper confidence bound (C-UCB) and causal Thompson Sampling (C-TS), and showed that they have improved cumulative regret bounds compared with algorithms that do not use causal information. They focus on causal relations among interventions and use causal graphs to capture the dependence among reward distribution of these interventions.

{\noindent\bf Fair Machine Learning.}
Fairness in machine learning has been a research subject with rapid growth and attention recently. Many different definitions of fairness have been designed to fit different settings, e.g., equality of opportunity and equalized odds \cite{hardt2016equality,zafar2017fairness}, direct and indirect discrimination \cite{zhang2017causal,zhang2018fairness,chiappa2018path}, counterfactual fairness \cite{kusner2017counterfactual,russell2017worlds,wu2019counterfactual}, and path-specific counterfactual fairness \cite{wu2019pcfairness}.
Related but different from our work include long term fairness (e.g., \cite{DBLP:conf/icml/LiuDRSH18}), which concerns for how decisions affect the long-term well-being of disadvantaged groups measured in terms of a temporal variable of interest, fair pipeline or multi-stage learning
(e.g., \cite{DBLP:journals/corr/BowerKNSVV17,DBLP:conf/ijcai/EmelianovAGGL19,DBLP:conf/innovations/DworkI19,DBLP:conf/forc/DworkIJ20}), which  primarily consider the combination of multiple non-adaptive sequential decisions and evaluate fairness at the end of the pipeline, and fair sequential learning (e.g., \cite{DBLP:conf/nips/JosephKMR16}), which  sequentially considers each individual and makes decision for them.
In \cite{DBLP:conf/icml/LiuDRSH18}, the authors proposed the study of delayed impact of fair machine learning and introduced a one-step feedback model of decision-making to  quantify the long-term impact of classification on different groups in the population.
\cite{DBLP:journals/corr/abs-2011-06738} developed a metric-free individual fairness and a cooperative contextual bandits (CCB) algorithm. The CCB algorithm utilizes fairness as a reward and attempts to maximize it. It tries to achieve individual fairness unlimited to problem-specific similarity metrics using multiple gradient contextual bandits.

\section{Conclusions}
In our paper, we studied how to learn optimal interventions sequentially by incorporating causal inference in bandits. We developed D-UCB and F-UCB algorithms which
leverage the d-separation set identified from the underlying causal graph and adopt soft intervention to model the arm selection strategy. Our F-UCB further achieves counterfactual individual fairness in each round of exploration by choosing arms from a subset of arms satisfying counterfactual fairness constraint. Our theoretical analysis and empirical evaluation show the effectiveness of our algorithms against baselines.

\begin{ack}
This work was supported in part by NSF 1910284, 1920920, 1940093, and 1946391.
\end{ack}

\newpage
\appendix
\section{Nomenclature and Assumptions}

In our regret bound analysis of D-UCB and F-UCB algorithms, we follow several standard assumptions \cite{DBLP:journals/corr/abs-2006-10185}  to guarantee the correctness and the simplicity of the proofs.

\begin{assumption}
For all $t \in [T]$,  both the error term of reward and the error term of counterfactual fairness discrepancy follow 1-sub-Gaussian distribution.
\end{assumption}

\begin{assumption}
For all $t \in [T]$, both the mean of  reward and the mean of counterfactual fairness discrepancy are within $[0,1]$. \end{assumption}

\begin{assumption}
There exists a safe policy $\pi_0$, i.e., $\pi_0 \in \Pi_t$ such that $\Delta_{\pi_0} \leq \tau$ for each $t \in [T]$.
\end{assumption}

The last assumption introduces the existence of a safe policy at each round, which plays an important role in the regret bound analysis of F-UCB. The nomenclature used for the proof part is shown in Table \ref{nomenclature}.

\begin{table*}
\caption{Nomenclature}
\begin{tabularx}{\textwidth}{p{0.22\textwidth}X}
\toprule
  $\mathcal{A}_t$ & Arm set at time $t$ \\
  $\A,\X,R$  & Arm features, user features, and reward \\
  $\W$  & d-separation set that separates $R$ from $(\A \cup \X ) \backslash \W$ \\
  $\Z$ & The difference between the d-separation set $\W$ and $\A\cup\X$. \\
  $\delta$ & With probability at least $1-\delta$ that the true reward is less than the estimated upper confidence bound for an arm $a$ at time $t$ \\
  $\delta_E$ & With probability at least $1-\delta_E$ that the regret of causal fair bandit is bounded \\
  $\delta'$ & With probability at least $1-\delta'$ that the true counterfactual discrepancy is less than its estimated upper confidence bound for an arm $a$ at time $t$ \\

  $UCB_a(t)$ & Upper confidence bound of the reward for action $a$ based on the observed values up to time $t$ \\
  $\mu_a$ & Expected mean reward for arm $a$\\
  $\mu_{a,s^*}$ & Estimated mean reward for arm $a$ if $gender = s^*$ given the user's profile\\
  $\mu_{\pi}$ & Expected mean reward for taking policy $\pi$\\
  $\hat{\mu}_{\pi}(t)$ & Estimated mean reward of a policy $\pi$ based on the observed values up to time $t$\\
  $\mathcal{R}_T$ & Cumulative regret up to time $T$ \\
  $\alpha_r$ & Parameter that controls the scale of the confidence interval of reward \\
  $\alpha_c$ & Parameter that controls the scale of the confidence interval of counterfactual discrepancy \\
  $\gamma_t$ & Parameter that could be tuned to ensure the fairness of a certain policy\\
  $E$ & Event under which all the true rewards are less than the estimated upper confidence bound \\
  $E_{cf}$ & Event under which all the counterfactual discrepancies are less than the estimated upper confidence bound \\
\bottomrule
\end{tabularx}
\label{nomenclature}
\end{table*}

\section{Regret Bound of D-UCB}
\label{proof_cucb}

\begin{theorem*}[Regret bound of D-UCB]
Given a causal graph $\mathcal{G}$, with probability at least $1-2\delta T |\mathbf{W}| - \exp(- \frac{|\mathbf{W}| \log^3(T)}{32\log(1/\delta)})$,  the regret of the causal bandit based on d-separation (D-UCB) algorithm is bounded by
\[  \mathcal{R}_T \leq \sqrt{|\mathbf{W}| T\log(T)}\log(T) + \sqrt{32|\mathbf{W}|T \log(1/\delta)}   \]
 where $|\mathbf{Z}|$  denotes the domain size of subset $\Z$, i.e., the difference between the d-separation set $\W$ and $\A\cup\X$ in $\mathcal{G}$.
\label{causalucb}
\end{theorem*}

\begin{proof}
Following the definition we can further define the expected reward mean of a certain policy as
\begin{align*}
  \mu_{\pi} = \mathbb{E}_{a \sim \pi}[\mu_a|do(a)] &= \mathbb{E}_{a \sim \pi} \left[ \sum_{i=1}^{|\mathbf{Z}|} \mathbb{E}[R|\mathbf{W} = \mathbf{w}_i ] P(\mathbf{Z} = \mathbf{z}_i| a) \right]
\end{align*}
and the policy applied at each time $t$ as
 $\pi_t  = \argmax_{\pi \in \Pi_t} \mathbb{E}_{a \sim \pi}[\mu_a]$.

Let $N_{\mathbf{w}}(t) = \sum_{s=1}^{t} \mathbbm{1}_{\mathbf{W}_s = \mathbf{w}}$  denote the count for a certain domain value of $\W$ up to time $t$.  Further we define the mean of the reward related to a d-separation set domain value as $\mu_{\w} =   \mathbb{E}[R|\mathbf{W} = \mathbf{w}]$ and its estimated value as $ \hat{\mu}_{\w}(t) =  \frac{1}{N_{\w}(t)} \sum^{t}_{s =1}  R_{a_s}   \mathbbm{1}_{\mathbf{W}_s = \mathbf{w}}$.

We also define the upper confidence bound of the reward for each arm and the upper confidence bound for each policy:

\[ UCB_{\mathbf{w}}(t) = \hat{\mu}_{\mathbf{w}}(t) + \sqrt{\frac{2\log(1/\delta)}{1 \vee N_{\mathbf{w}}(t)}}  ~~~~~~~~~~~~ UCB_{a}(t) = \sum_{\Z} UCB_{\mathbf{w}}(t) P(\mathbf{z}|\x_{t,a})
\]

\[ \mathbb{E}_{a \sim \pi} [UCB_{a}(t)] = \mathbb{E}_{a \sim \pi} \left[ \sum_{\Z} UCB_{\mathbf{w}}(t) P(\mathbf{z}|\x_{t,a}) \right]  \]

Let $E$ be the event that for all time $t \in [T]$ and value index $i  \in [|\mathbf{W}|]$, we have
\[ |\hat{\mu}_{\w_i}(t) - \mu_{\w_i}| \leq \sqrt{\frac{2\log(1/\delta)}{1 \vee N_{\w_i}(t)}} \]
 Since $\hat{\mu}_{\w_i}(t)$  is the sample mean estimator of $\mu_{\w_i}$, and the error term follows sub-Gaussian distribution, we can show
\begin{align*}
 & P\left( |\hat{\mu}_{\w_i}(t) - \mu_{\w_i}| \geq \sqrt{\frac{2\log(1/\delta)}{1 \vee N_{\w_i}(t)}} \right)  = \mathbb{E}\left[ P \left( |\hat{\mu}_{\w_i}(t) - \mu_{\w_i}| \geq \sqrt{\frac{2\log(1/\delta)}{1 \vee N_{\w_i}(t)}} \middle| \w_{(1)},...,\w_{(t)} \right)\right] \\
&~~~~~~~~~~~~~~~~~~~~~~~~~~~~~~~~~~~~~~~~~~~~~~~~~~~~~~~~~~~~~\leq \mathbb{E}[2\delta] = 2\delta
\end{align*}
where $\w_{(t)}$ denotes the observed values at time $t$. Thus by summing up the probabilities through all domain values of $t \in [T]$ and $i \in [|\mathbf{W}|]$, using union bound criteria we have $P(E^c) = 1 - P(E) \leq 2 \delta T |\mathbf{W}|$. The above result implies a lower probability bound for event $E$. The cumulative regret could be decomposed as
\begin{align*}
\mathcal{R}_T = \sum_{t=1}^{T}(\mu_{a^*} - \mu_{a_t})
    = \sum_{t=1}^{T}\left(\mu_{a^*} - UCB_{a_t}(t) + UCB_{a_t}(t) - \mu_{a_t}\right)
\end{align*}
Following the rule of optimism in the face of uncertainty, under event $E$ we have
\begin{align*}
\mu_{a^*}  &= \sum_{i=1}^{|\mathbf{Z}|} \mathbb{E}[R|\mathbf{W} = \mathbf{w}_i] P(\mathbf{Z} = \mathbf{z}_i|a^*) \\
& \leq   \sum_{i=1}^{|\mathbf{Z}|} UCB_{\w_i}(t) P(\mathbf{Z} = \mathbf{z}_i|a^*) = UCB_{a^*}(t)
\end{align*}
As $UCB_{a^*}(t) \leq UCB_{a_t}(t)$ always holds due to OFU arm picking strategy, we have $\mu_{a^*} - UCB_{a_t}(t) \leq 0$.

With probability at least $1 - 2 \delta T |\mathbf{W}|$, the cumulative regret can thus be further bounded by
\begin{align}
\mathcal{R}_T &\leq \sum_{t=1}^{T}( UCB_{a_t}(t)  - \mu_{a_t}) \nonumber \\
    &= \sum_{t=1}^{T}\sum_{i=1}^{|\mathbf{Z}|} (UCB_{\w_i}(t) -\mu_{\w_i})  P(\mathbf{Z} = \mathbf{z}_i|a_t) \nonumber  \\
    &\leq \sum_{t=1}^{T}\sum_{i=1}^{|\mathbf{W}|} \sqrt{\frac{8\log(1/\delta)}{1 \vee N_{\w_i}(t)}}  P(\mathbf{Z} = \mathbf{z}_i|a_t) \nonumber \\
&= \sum_{t=1}^{T}\sum_{i=1}^{|\mathbf{W}|} \sqrt{\frac{8\log(1/\delta)}{1 \vee N_{\w_i}(t)}} \left( P(\mathbf{Z} = \mathbf{z}_i|a_t) - \mathbbm{1}_{Z_{(t)} = Z_i} \right) +\sum_{t=1}^{T}\sum_{i=1}^{|\mathbf{W}|} \sqrt{\frac{8\log(1/\delta)}{1 \vee N_{\w_i}(t)}} \left(\mathbbm{1}_{Z_{(t)} = Z_i} \right)
\label{regret_latter}
\end{align}

The second part of Equation \ref{regret_latter} is bounded by
\begin{align*}
 \sum_{t=1}^{T}\sum_{i=1}^{|\mathbf{W}|}\sqrt{\frac{8\log(1/\delta)}{1 \vee N_{\w_i}(t)}} \mathbbm{1}_{Z_{(t)} = Z_i}  &\leq \sum_{i=1}^{|\mathbf{W}|} \int_{0}^{N_{\w_i}(T)} \sqrt{\frac{8\log(1/\delta)}{s}}ds \\
&\leq \sum_{i=1}^{|\mathbf{W}|} \sqrt{32N_{\w_i}(T)\log(1/\delta)}  \\
&\leq \sqrt{32|\mathbf{W}| T \log(1/\delta)}
\end{align*}

We will use the following proposition called Azuma's inequality to derive the bound of the first term of Equation \ref{regret_latter}.
\begin{proposition}
Suppose $\{ M_k: k = 0,1,2...   \} $  is a martingale and $|M_k - M_{k-1}| < c_k$ almost surely, then for all $t \in [T]$ and positive value $\epsilon$ we have:
\[ P(|M_t-M_0| > \epsilon ) \leq  \exp \left( \frac{-\epsilon^2}{2\sum_{k=1}^{t} c_k^2} \right)  \]
\end{proposition}

For the first part, we further define \[M_t = \sum_{s=1}^{t}\sum_{i=1}^{|\mathbf{W}|}\sqrt{\frac{8\log(1/\delta)}{1 \vee N_{\w_i}(t)}}\left( P(\mathbf{Z} = \mathbf{z}_i|a_t) - \mathbbm{1}_{Z_{(t)} = Z_i}\right)\]
with $M_0 = 0$, since $\{ M_t\}_{t=0}^T$ is a martingale sequence, we have
\begin{align*}
 |M_t - M_{t-1}|^2  &=  \left| \sum_{i=1}^{|\mathbf{W}|}  \sqrt{\frac{8\log(1/\delta)}{1 \vee N_{\w_i}(t)}} \left( P(\mathbf{Z} = \mathbf{z}_i|a_t)  - \mathbbm{1}_{Z_{(t)} = Z_i} \right) \right|^2  \leq 32\log(1/\delta)
\end{align*}
which shows $|M_t - M_{t-1}|$ is bounded for any $t \in [T]$. Applying Azuma's inequality, we have
\begin{align*}
P\left(|M_T-M_0|> \sqrt{|\mathbf{W}| Tlog(T)}\log(T)\right)  &= P\left(|M_T|> \sqrt{|\mathbf{W}| T\log(T)}\log(T)\right)\\
&\leq \exp(- \frac{|\mathbf{W}| \log^3(T)}{32\log(1/\delta)})
\end{align*}

The formula above gives a high probability bound of the first part. Now we can combine the bounds of two parts in Equation \ref{regret_latter} to derive the high probability bound of $\mathcal{R}_T$. Since $P(E^c) \leq 2\delta T |\mathbf{W}|$, applying union bound rule, with probability at least $1-2\delta T |\mathbf{W}|- \exp(- \frac{|\mathbf{W}| \log^3(T)}{32\log(1/\delta)})$, the regret is bounded by:
\begin{equation}
    \mathcal{R}_T \leq \sqrt{|\mathbf{W}| T\log(T)}\log(T) + \sqrt{32|\mathbf{W}| T \log(1/\delta)}
\label{high_prob_bound}
\end{equation}
\end{proof}

\begin{corollary}
By setting $\delta = 1/T^2$, the causal bandit algorithm achieves $\Tilde{O}(\sqrt{|\mathbf{W}| \cdot T})$ regret bound.
\end{corollary}

\begin{proof}
Plugging in the value $\delta = 1/T^2$, with probability at least $ 1-2|\mathbf{W}|/T-\exp(-\frac{|\mathbf{W}| \log^2(T)}{64}) $, the regret is bounded by
\[ R_T \leq 16\sqrt{|\mathbf{W}| T\log(T)}\log(T) \]

The above formula thus leads to $\Tilde{O}(\sqrt{|\mathbf{W}| \cdot T})$ long-term expected regret.

\end{proof}

\section{Upper Bound of Unidentifiable Counterfactual Fairness}
\begin{theorem*}
Given a causal graph as shown in Figure 1 in the main paper, if there exists a non-empty set $\mathbf{B} \subseteq \X \backslash \{S\} $ which are descendants of $S$, then
$\mu_{a,s^*}=\mathbb{E}[R(a,s^*)|\mathbf{x}_{t}]$ is bounded by
\begin{equation}
\sum_{\Z}\min_{\mathbf{b}}\{ \mathbb{E}[R|s^{*},\mathbf{w}\backslash s_t] \} \cdot P(\z|\mathbf{x}_{t,a}) \leq \mu_{a,s^*} \leq \sum_{\Z}\max_{\mathbf{b}}\{ \mathbb{E}[R|s^{*},\mathbf{w}\backslash s_t] \} \cdot P(\z|\mathbf{x}_{t,a})
\end{equation}
\end{theorem*}

\begin{proof}

If a set of attributes $\mathbf{B} \subseteq \X \backslash \{S\} $ are descendants of $S$, $\mathbb{E}[R(a,s^*)|\mathbf{x}_{t}]$ is not identifiable. According to Proposition 2 in \cite{wu2019counterfactual}, we have that
\begin{equation*}
    \mathbb{E}[R(a,s^*)|\mathbf{x}_{t}] \leq \sum_{\mathbf{I}} \frac{P(\x_{t,a},\mathbf{i})}{P(\x_{t,a})} \max_{\mathbf{b}} \{ \mathbb{E}[R|s^{*},\x_{t,a}\backslash s_{t}, \mathbf{i}] \}
\end{equation*}
It follows that
\begin{equation*}
\begin{split}
     \mathbb{E}[R(a,s^*)|\mathbf{x}_{t}] &\leq \sum_{\mathbf{I}} P(\mathbf{i}|\x_{t,a}) \max_{\mathbf{b}} \{ \mathbb{E}[R|s^{*},\x_{t,a}\backslash s_{t}, \mathbf{i}] \} \\
    & = \sum_{\mathbf{Z},\mathbf{I}\backslash \mathbf{Z}} P(\z|\x_{t,a})P(\mathbf{i}\backslash \z|,\z,\x_{t,a}) \max_{\mathbf{b}} \{ \mathbb{E}[R|s^{*},\w\backslash s_{t}] \},
\end{split}
\end{equation*}
where $\mathbf{Z}$ and $\W$ are defined following Section 3.2 in the main paper.
We claim that $\mathbf{W}$ has no intersection with $\mathbf{I}\backslash \Z$. Otherwise, there exists an attribute $I\in \mathbf{I}$ which belongs to $\W$ but not $\Z$. This contradicts to the definition of $\Z$, which is given by $\W$ subtracting $\A\cup \X$. Thus, it follows that
\begin{equation*}
\begin{split}
    \mathbb{E}[R(a,s^*)|\mathbf{x}_{t}] &\leq \sum_{\mathbf{Z}} \max_{\mathbf{b}} \{ \mathbb{E}[R|s^{*},\w\backslash s_{t}] \} P(\z|\x_{t,a}) \sum_{\mathbf{I}\backslash \mathbf{Z}}P(\mathbf{i}\backslash \z|,\z,\x_{t,a})  \\
    & =  \sum_{\mathbf{Z}} \max_{\mathbf{b}} \{ \mathbb{E}[R|s^{*},\w\backslash s_{t}] \} P(\z|\x_{t,a})
\end{split}
\end{equation*}

\end{proof}

\section{Regret Bound of F-UCB}
\label{proof_fucb}

\begin{theorem*}[Regret bound of fair causal bandit]
Given a causal graph $\mathcal{G}$, let $\delta_E = 4|\mathbf{W}|T\delta$ and $\Delta_{\pi_0}$  denote the fairness discrepancy of the safe policy $\pi_0$. Setting $\alpha_c = 1$ and $\alpha_r = \frac{2}{\tau - \Delta_{\pi_0}}$, with probability at least $1-\delta_E$, the cumulative regret of the F-UCB algorithm is bounded by:
\begin{align*}
   \mathcal{R}_T &\leq (\frac{2}{\tau - \Delta_{\pi_0}}+1) \times
   \left( 2\sqrt{2T|\mathbf{W}|\log(1/\delta_E)} + 4\sqrt{T\log(2/\delta_E)\log(1/\delta_E)} \right)
\end{align*}
\end{theorem*}

\begin{proof}
Similar to the regret analysis of causal bandit, we decompose the cumulative regret $\mathcal{R}_T$ into two parts.
\begin{align}
    \mathcal{R}_T &= \sum_{t=1}^{T}\left(\mathbb{E}_{a \sim \pi^*}[\mu_{a}] - \mathbb{E}_{a \sim \pi_t}[\mu_{a}]\right)  \nonumber \\
        &= \left( \sum_{t=1}^{T}\mathbb{E}_{a \sim \pi^*}[\mu_{a}] - \mathbb{E}_{a \sim \pi_t}[UCB_{a}(t)]  \right)  +  \left( \sum_{t=1}^{T} \mathbb{E}_{a \sim \pi_t}[UCB_{a}(t)] - \mathbb{E}_{a \sim \pi_t}[\mu_{a}] \right)
\label{rt}
\end{align}

We will further bound $\mathcal{R}_T$  by proving the first term of Equation \ref{rt} is less than 0 and the second term could be bounded by adopting the upper confidence bound analysis approach. In the fair bandit setting we introduce another event $ E_{cf}$ that implies the estimation error of the counterfactual discrepancy is bounded. That is, for all time $t \in [T]$ and a policy $\pi$, with probability at least $1-\delta'$,
\[|\hat{\Delta}_{\pi}(t) - \Delta_{\pi}| \leq \sqrt{\frac{8\log(1/\delta')}{1 \vee N_{a}(t)}} \]

First we define the inflated upper confidence bound with scale parameters for the mean reward and fairness discrepancy as
\begin{align*}
UCB_a(t) &= \hat{\mu}_a + \alpha_r\beta_a(t) \\
 UCB_{\Delta_\pi}(t) = \hat{\Delta}_{\pi} + \alpha_c\beta_a(t)&,~\textit{where}~ \beta_a(t) = \sqrt{2\log(1/\delta')/N_a(t)}
\end{align*}

Notice that  the event $E_{cf}$ will always happen if the event $E$ happens. Under event $E \cap  E_{cf} = E$
and with the assumption that $\alpha_r, \alpha_c \geq 1$, we have
\begin{align}
 (\alpha_r -1)\beta_a(a) \leq \epsilon^r_a(t)  \leq (\alpha_r + 1)\beta_a(a),  ~\forall a \in \mathcal{A}_t \nonumber \\
 (\alpha_c -1)\beta_a(a) \leq \epsilon^c_a(t)  \leq (\alpha_c + 1)\beta_a(a)  ,  ~\forall a \in \mathcal{A}_t
 \label{error_range}
\end{align}
where $\epsilon^r_a(t) $  and $\epsilon^c_a(t) $ are the error term of the reward and counterfactual discrepancy.
If the optimal policy belongs to the fair policy subspace, which means $\pi^* \in \bar{\Phi}_t$, we can easily get:
\[  \mathbb{E}_{a \sim \pi^*}[\mu_{a}]  \leq  \mathbb{E}_{a \sim \pi^*}[UCB_{a}(t)]   \leq  \mathbb{E}_{a \sim \pi_t}[UCB_{a}(t)] \]

Now we assume $\pi^* \notin \bar{\Phi}_t$,  that is, $\mathbb{E}_{a_t \sim \pi^*}[UCB_{\Delta_{\pi^*}}(t)]  > \tau$. Let $\pi^* = \rho^* \bar{\pi}^* + (1 - \rho^*)\pi_0 $, where $\bar{\pi}^*$ denotes the optimal policy in the policy subspace $\bar{\Phi}_t \setminus \pi_0$.

Consider the mixed policy of $\pi^*$ and $\pi_0$, denoted as $\tilde{\pi} = \gamma_t \pi^* + (1-\gamma_t)\pi_0 = \gamma_t \rho^*\bar{\pi}^* + (1-\gamma_t\rho^*)\pi_0 $,  where $\gamma_t \in [0,1]$ is the maximum value to ensure $\tilde{\pi}
\in \Phi_t$. One feasible solution for $\gamma_t$ is
\begin{align}
\gamma_t &= \frac{\tau - \Delta_{\pi_0}}{\rho^* \mathbb{E}_{a_t \sim \bar{\pi}_t}[UCB_{\Delta_{\bar{\pi}_t}}(t)] - \rho^*\Delta_{\pi_0}} \nonumber \\
&= \frac{\tau - \Delta_{\pi_0}}{ \mathbb{E}_{a_t \sim \bar{\pi}_t}[\rho^*(\bar{c}_a+ \epsilon_a^c(t))] - \rho^*\Delta_{\pi_0}} \nonumber \\  \nonumber
         &\geq \frac{\tau - \Delta_{\pi_0}}{ \tau - \Delta_{\pi_0} + \rho^*(1+\alpha_c) \mathbb{E}_{a_t \sim \bar{\pi}^*}[\beta_a(t)] }
\label{gamma}
\end{align}
Denote $\frac{\tau - \Delta_{\pi_0}}{\tau -\Delta_{\pi_0} + \rho^* (1+ \alpha_c)\mathbb{E}_{a \sim \bar{\pi}^*}[\beta_a(t)]}$ as $\Gamma$. From the design of the integrated policy $\tilde{\pi}_t$ we further have:
\begin{align}
\mathbb{E}_{a \sim \pi_t}[UCB_{a}(t)] &\geq \gamma_t \mathbb{E}_{a_t \sim \pi^*}[UCB_{a}(t)] + (1-\gamma_t)UCB_{a}(t)       \\
&\geq \Gamma \times \mathbb{E}_{a_t \sim \pi^*}[UCB_{a}(t)]  \label{first} \\
&= \Gamma \times \left( \mathbb{E}_{a_t \sim \pi^*}[\mu_a] +  \mathbb{E}_{a_t \sim \pi^*}[\epsilon_{a_t}]  \right) \nonumber \\
&\geq \Gamma \times \left( \mathbb{E}_{a_t \sim \pi^*}[\mu_a] +  (\alpha_r -1)\mathbb{E}_{a_t \sim \pi^*}[\beta_t]  \right)  \label{second} \\
&\geq  \frac{\tau - \Delta_{\pi_0}}{\tau -\Delta_{\pi_0} + (1+ \alpha_c)\mathbb{E}_{a \sim \pi^*}[\beta_a(t)]} \times
\left( \mathbb{E}_{a_t \sim \pi^*}[\mu_a] +  (\alpha_r -1)\mathbb{E}_{a_t \sim \pi^*}[\beta_t]  \right) \label{c0}
\end{align}

For the above derivation, Equation \ref{first} holds because $UCB_{a}(t) > 0$, Equation \ref{second} is the consequence of Equation \ref{error_range}, Equation \ref{c0} is derived based on the fact that $\mathbb{E}_{a_t \sim \pi^*}[\beta_t] = \rho^*\mathbb{E}_{a_t \sim \bar{\pi}^*}[\beta_t] + (1-\rho^*\beta_0(t)) \geq \rho^* \mathbb{E}_{a_t \sim \bar{\pi}^*}[\beta_t] $.
Denote the term in Equation \ref{c0} as $C_0$. Let $C_1 = \mathbb{E}_{a \sim \pi^*}[\beta_a(t)]$, it holds that
\[ C_0 = \frac{\tau - \Delta_{\pi_0}}{\tau - \Delta_{\pi_0} + (1+ \alpha_c)C_1} \times \left( \mathbb{E}_{a_t \sim \pi^*}[\mu_a] +  (\alpha_r -1)C_1  \right)\]
$C_0 > \mathbb{E}_ {a \sim \pi^*}[\mu_a]$ is satisfied if and only if:
\[ (\tau - \Delta_{\pi_0})(\alpha_r -1)C_1 \geq (1 + \alpha_c)C_1\mathbb{E}_ {a \sim \pi^*}[\mu_a] \]

Since $\mathbb{E}_ {a \sim \pi^*}[\mu_a] \leq 1$, the above inequation holds if $ (\tau - \Delta_{\pi_0})(\alpha_r -1)C_1 \geq 1 + \alpha_c $. Thus, by setting $\delta_E = 4|\mathbf{W}|T\delta$ and $\alpha_r, \alpha_c \geq 1, \alpha_c \leq \tau(\alpha_r -1) $, following simple union bound rule implies that with probability at least $1- \frac{\delta_E}{2} $, we have
\[ \sum_{t=1}^{T}\left( \mathbb{E}_{a \sim \pi^*}[\mu_a] - \mathbb{E}_{a_t \sim \pi_t}[UCB_{a}(t)]\right) \leq 0  \]

Next we will derive the bound of the second term in Equation \ref{rt}. The result is given by the following proposition.
\begin{proposition}
If $\delta_E = 4|\mathbf{W}|T\delta$ for a $\delta \in (0,1)$, then with probability at least $1-\frac{\delta_E}{2}$, we have
\begin{align}
\sum_{t=1}^{T} \mathbb{E}_{a \sim \pi_t}[UCB_{a}(t)] - \mathbb{E}_{a \sim \pi_t}[\mu_a] \leq \nonumber (\alpha_r + 1) \left(2\sqrt{2T|\mathbf{W}|\log(1/\delta)} + 4\sqrt{T\log(2/\delta_E)\log(1/\delta)} \right)
\end{align}

\label{latter term}
\end{proposition}

Under the conditions in the proposition, we have $P(E) \geq 1- \frac{\delta_E}{2}$. Under the event $E$, we have $R_{a_t} \in [\hat{\mu}_{a_t}-\beta_a(t)  , \hat{\mu}_{a_t} + \beta_a(t)  ]$ for all $t \in [T]$ and $a \in \mathcal{A}$. Thus for all $t$ we could further derive
\[   \mathbb{E}_{a \sim \pi_t}[UCB_{a}(t)] -  \mathbb{E}_{a \sim \pi_t}[\mu_{a}] \leq (\alpha_r +1 )  \mathbb{E}_{a \sim \pi_t}[\beta_a(t)]  \]

Let $\mathcal{F}_{t-1}$ be the $\sigma$-algebra defined up to the choice of $\pi_t$ and $a_t'$ be another choice picked from policy $\pi_t|\mathcal{F}_{t-1} $. $a_t'$ is conditionally independent of $a_t$, which means $a_t' \independent a_t | \mathcal{F}_{t-1}$. By definition the following equality holds:
\[   \mathbb{E}_{a \sim \pi_t}[\beta_a(t)] = \mathbb{E}_{a'_t \sim \pi_t}[\beta_a(t) | \mathcal{F}_{t-1}] \]

Setting $A_t = \mathbb{E}_{a'_t \sim \pi_t}[\beta_a(t) | \mathcal{F}_{t-1}] - \beta_{a_t}(t)$, $M_t = \sum_{s=1}^{t} A_t$ is thus a martingale sequence with $|M_t - M_{t-1}| = |A_t| \leq 2\sqrt{2\log(1/\delta)}$. Thus applying Azuma-Hoeffding inequality implies:

\begin{align*}
 &P\left( \sum_{t=1}^{T} \mathbb{E}_{a \sim \pi_t}[\beta_a(t)] \geq \sum_{t=1}^{T} \beta_{a_t}(t) + 4\sqrt{T\log(2/\delta_E)\log(1/\delta)} \right) \leq \delta_E/2
\end{align*}

We denote the event that describes the results of the above inequality  as $E_A$. In the equation above, the sum of the adaptive scaling parameter could be decomposed as follows:
\[  \sum_{t=1}^{T} \beta_{a_t}(t) = \sum_{a \in \mathcal{A}}\sum_{t=1}^{T}\mathbbm{1}_{ \{a_t =a\} } \beta_{a}(t) = \sum_{i \in |\mathbf{W}|}\sum_{t=1}^{T}\mathbbm{1}_{ \{\mathbf{W}_t =\w_i\} } \beta_{\w_i}(t)\] Under event $E$, for each domain value of the d-separation set $\W$ we have:
\begin{align*}
\sum_{t=1}^{T}\mathbbm{1}_{ \{\mathbf{W}_t =\w_i\} } \beta_{\w_i}(t) & =  \sqrt{2\log(1/\delta)} \sum_{t=1}^{N_{\w_i}(T)}\frac{1}{\sqrt{t}} \leq 2\sqrt{2N_{\w_i}(T)\log(1/\delta)}
\end{align*}

Since $\sum_{i \in |\mathbf{W}|} N_{\w_i}(T) = T $, using the fact that arithmetic mean is less than quadratic mean we have:
\[ \sum_{i \in |\mathbf{W}|}2\sqrt{2N_{\w_i}(T)\log(1/\delta)}  \leq  2\sqrt{2T|\mathbf{W}|\log(1/\delta)}\]
Conditioning on the event $E \cap E_A$ whose probability satisfies $P(E \cap E_A) \geq 1 - \delta_E$, we have
\begin{align*}
P\bigg( \mathbb{E}_{a \sim \pi_t}[UCB_{a}(t)] - \mathbb{E}_{a \sim \pi_t}[\mu_{a}] &\geq (\alpha_r+1)
\left( 2\sqrt{2T|\mathbf{W}|\log(1/\delta)} + 4\sqrt{T\log(2/\delta_E)\log(1/\delta)} \right) \bigg) \\
&\leq \delta_E/2
\end{align*}

which is exactly the result of Proposition \ref{latter term}.

Finally, combining the theoretical derivation of the two parts above leads to the cumulative regret bound shown in Theorem \ref{cfbound}.
\end{proof}

\section{Experimental Settings}

Our experiments were carried out on a Windows 10 Enterprise  workstation with a 3.2 GHz
Intel Core i7-8700 CPU and 64GB RAM.

\subsection{Setting of Email Campaign Data}
\label{Email Campaign table}
Table \ref{tab:domain value} shows attributes of Email Campaign data and their domain values.
Table \ref{I4} shows the conditional probabilities of $ P(I_4 = i|X_1,X_2,X_3)$.  The following equations are the conditional distributions for the remaining variables.
\begin{align*}
    P(I_2=1|A_1,A_2,I_4) &=  (A_1+A_2+I_4)/12 \\
    P(I_1=1|A_1,A_2,I_2) &=  (A_1+A_2+I_2)/10 \\
    P(I_3 = 1|I_2 = 1) &= 0.4 \\
    P(I_3 = 1|I_2 = 2) &= 0.6 \\
\end{align*}

 \begin{table}[!h]
\centering
 \caption{Variables in Email campaign data}
  \begin{tabular}{lll}
    \toprule
    \cmidrule(r){1-2}
    Variables     & Domain Value \\
    \midrule
    Click ($R$) &  $(0,1)$ \\
    Gender ($X_1$) & $(1,2)$ \\
    Age ($X_2$)    & $(1,2)$   \\
    Occupation ($X_3$)    & $(1,2)$   \\
    Product ($A_1$) & $(1,2,3)$ \\
    Propose ($A_2$)    & $(1,2,3,4)$   \\
    Send time ($A_3$)    & $(1,2,3)$   \\
    Email body template ($I_1$) & $(1,2)$ \\
    Fitness ($I_2$)    & $(1,2)$   \\
    Subject length ($I_3$)    & $(1,2,3,4)$   \\
    User query ($I_4$)    & $(1,2)$   \\
    \bottomrule
  \end{tabular}
  \label{tab:domain value}
\end{table}

\begin{table}[h]
\centering
 \caption{Conditional probabilities of $ P(I_4 = i|X_1,X_2,X_3)$}
  \begin{tabular}{|c|c|c|c|c|} \hline

    ($X_1$, $X_2$, $X_3$)      & $i = 1$ & $i = 2$  & $i = 3$ & $ i = 4$\\ \hline
     (0,0,0) & 0.4  & 0.3  & 0.2 & 0.1\\
     (0,0,1) & 0.3  & 0.4  & 0.2 & 0.1\\
     (0,1,0) & 0.6  & 0.1 & 0.2  & 0.1 \\
     (0,1,1) & 0.5  & 0.2 & 0.2  & 0.1 \\
     (1,0,0) & 0.1 & 0.3 & 0.2   & 0.4 \\
     (1,0,1) & 0.1  & 0.4 & 0.2  & 0.3 \\
     (1,1,0) & 0.1  & 0.1 & 0.2  & 0.6  \\
     (1,1,1) & 0.1  & 0.2 & 0.2  & 0.5 \\ \hline
  \end{tabular}
  \label{I4}
\end{table}

\subsection{Setting of Adult-Video Data}
\label{Adult Video Section}
Following the setting of \cite{huang2020achieving}, we generate one simulated dataset for our experiments by combining the following two publicly available datasets.
\begin{itemize}
    \item \textbf{Adult dataset}: The Adult dataset \cite{Dua:2019} is used to represent the students (or bandit players). It is composed of 31,561 instances: 21,790 males and 10,771 females, each having 8 categorical variables (work class, education, marital status, occupation, relationship, race, sex, native-country) and 3 continuous variables (age, education number, hours per week). We select 4 variables, \textit{age, sex, race, income}, as user features in our experiments and binarize their domain values due to data sparsity issue.

    \item \textbf{YouTube dataset}: The Statistics and Social Network of YouTube Videos \footnote{https://netsg.cs.sfu.ca/youtubedata/} dataset is used to represent the items to be recommended (or arms). It is composed of 1,580 instances each having 6 categorical features (age of video, length of video, number of views, rate, ratings, number of comments). We select four of those variables (age, length, ratings, comments) and binarize them for a suitable size of the arm pool.

\end{itemize}

For our experiments, we use a subset of 10,000 random instances from the Adult dataset, which is then split into two subsets: one for graph construction and the other for online recommendation.
Similarly, a subset of YouTube dataset is used as our pool of videos to recommend. The  subset contains 16 video types (arms) representing different domain values of the 4 binarized arm features.

The feature contexts  $\mathbf{x_{t,a}}$ used throughout the experiment is the concatenation of both the student feature vector and the video feature vector. Four elements in $\mathbf{x_{t,a}}$ are selected according to domain knowledge as the variables that will determine the value of the reward. A linear reward function is then applied to build this mapping relation from those selected variables to the reward variable.  In our experiments we choose the sensitive attribute to be the \textbf{gender of adults}, and focus on the individual level fairness discrepancy regarding to both male and female individuals.

For the email campaign experiment setting, we construct the causal graph following the domain knowledge and one of the recent research works on causal bandit \cite{lu2020regret}. For the Adult-Video experiment setting, we construct the causal graph using a causal discovery software Tetrad (https://www.ccd.pitt.edu/tools/).

\clearpage

\bibliographystyle{plain}
\bibliography{paper}

\end{document}